\documentclass[letterpaper, 10 pt, conference]{ieeeconf}  

\IEEEoverridecommandlockouts                              

\usepackage{float,amsmath,amsfonts, color,graphicx}

\usepackage{amssymb, amsthm}
\usepackage{calc}
\usepackage{tikz}
\usepackage{pgfplots}
\usepackage{cases}
\usepackage{url}
\usepackage[ruled, vlined, linesnumbered]{algorithm2e}
\usepackage{amsthm}
\usepackage{array}
\usepackage{tabularx}
\usepackage{diagbox}
\usepackage{subcaption}

\title{\LARGE \bf
Synthesis of Control Barrier Functions Using a Supervised Machine Learning Approach
}

\author{Mohit Srinivasan$^{1}$, Amogh Dabholkar$^{2}$, Samuel Coogan$^{3}$, and Patricio Vela$^{4}$%
\thanks{This work was supported in part by the National Science Foundation under Grant Nos. \#1836932, S\&AS 1849333, and in part by DARPA PAI}%
\thanks{$^{1}$Mohit Srinivasan, $^{3}$Samuel Coogan and $^{4}$Patricio Vela are with the School of Electrical and Computer Engineering,
        Georgia Institute of Technology, Atlanta, USA
        {\tt\small mohit.srinivasan@gatech.edu; sam.coogan@gatech.edu; pvela@gatech.edu}}
\thanks{$^{2}$Amogh Dabholkar is with the Department of Electrical and Electronics Engineering, Birla Institute of Technology and Science (BITS), Pilani - K. K. Birla Goa Campus, India
        {\tt\small adabholkar6@gatech.edu}}
\thanks{$^{3}$Samuel Coogan is also with the School of Civil and Environmental Engineering, Georgia Institute of Technology, Atlanta, USA}}

\theoremstyle{plain}
\newtheorem{theorem}{Theorem}
\newtheorem{prop}{Proposition}
\newtheorem{probstat}{Problem Statement}

\theoremstyle{definition}

\newtheorem{definition}{Definition}

\DeclareMathAlphabet{\mathcal}{OMS}{cmsy}{m}{n}

\newcommand{\kfun}{k_{\phi}}

\begin{document}

\maketitle
\thispagestyle{empty}
\pagestyle{empty}

\begin{abstract}
Control barrier functions are mathematical constructs used to guarantee safety for robotic systems. When integrated as constraints in a quadratic programming optimization problem, instantaneous control synthesis with real-time performance demands can be achieved for robotics applications.  Prevailing use has assumed full knowledge of the safety barrier functions, however there are cases where the safe regions must be estimated online from sensor measurements. In these cases, the corresponding barrier function must be synthesized online.  This paper describes a learning framework for estimating control barrier functions from sensor data.  Doing so affords system operation in unknown state space regions without compromising safety. Here, a support vector machine classifier provides the barrier function specification as determined by sets of safe and unsafe states obtained from sensor measurements.  Theoretical safety guarantees are provided. Experimental ROS-based simulation results for an omnidirectional robot equipped with LiDAR demonstrate safe operation.

\end{abstract}

\section{INTRODUCTION}
\label{sec:intro}
Autonomous vehicles \cite{autonomous_vehicles}, industrial robots, and multi-robot systems \cite{mrs_safety} deployed in uncertain domains are often tasked to respect safety-critical constraints while advancing a given task \cite{safety_critical}.  When operating in unknown and dynamic environments with insufficient advanced information regarding the workspace, controllers which translate sensory information from the environment into \emph{safe} control actions are of paramount importance.  Control barrier functions (CBFs) are level-set functions used to provide formal safety guarantees for controlled dynamical systems.  Given a possibly unsafe nominal control policy, barrier function based quadratic programs (QPs) generate a \emph{safe} control action at each time instant. The control actions are minimally invasive in the sense that the nominal control policy is modified only when it will result in violation of a safety constraint.  Barrier function based real-time controllers in robotics support collision avoidance for multi-robot motion \cite{Li_multi_TRO}, task allocation for robotic swarms \cite{notomista2019optimal}, and motion planning \cite{mohit_cdc18}. 

A key assumption commonly imposed is that the robotic system has complete knowledge of the unsafe state space regions. Leveraging the knowledge translates to formal safety guarantees arising from its translation to CBFs.  In practice, this assumption need not hold and limits more widespread application of barrier functions. As a motivating example, consider an autonomous robot operating in an environment for which it has no knowledge of the obstacle boundaries. If these boundaries are to be as level-sets of {smooth functions}, the process of finding closed-form barrier functions for these obstacles is not straightforward.  Without the functions, one cannot leverage the safety guarantees that CBFs provide. Thus, this paper describes a support vector machine (SVM) approach to CBF synthesis from sensor measurements.  In particular, sensory information obtained from the environment defines the set of safe and unsafe samples and is used for training the SVM classifier.

Learning algorithms, or data-driven synthesis methods, for ensuring safety have been explored in several contexts. The most prevalent has been to establish stable state space regions meeting safety specifications by identifying a control Lyapunov function (CLFs) compatible with given CBFs. Techniques for doing so include sum-of-squares (SOS) methods \cite{Li_SOS} and neural network designs \cite{richards2018lyapunov}, with the aim of identifying the largest possible stable region within the safe region. When attempting to learn baseline control policies for a given task, reinforcement learning methods cannot guarantee safety as the exploration process demands executing unsafe control inputs.  Employing pre-specified barrier functions during the action policy exploration and keeping track of the safety interventions to influence the explored policies, can provide the necessary safety guarantees \cite{cheng2019RL_barriers}.  Investigations more closely aligned with barrier function synthesis using tools from machine learning include the use of kernel machines \cite{cristianini2000introduction} to synthesize occupancy map functions for navigation and planning purposes \cite{RaOt_HilbertMaps,FrOtRa_PathPlanHSOM}. Occupancy map level-sets can distinguish safe and unsafe regions.  This potential use was further explored in the context of perceptron algorithms, where the resulting classifier function provided a mechanism to synthesize non-colliding trajectories through space \cite{yip_kernel}.  Emphasis was on improving the run-time of global mapping relative to existing kernel machine methods.  Our aim is to explore how these machine learning constructs can be used to synthesize CBFs in a manner that the learned function provides the necessary safety guarantees.

The contributions of this work are as follows: First, we present a SVM approach for the synthesis of a barrier function from a training dataset consisting of \emph{safe} and \emph{unsafe} samples obtained from sensor measurements. We describe offline and online training methods.  Second, a formal guarantee on correct classification of unsafe regions is provided for both the methods. We show that in the offline method, the system is rendered safe for an under-approximated (conservative) safe set.  A similar guarantee holds locally in the online approach.  The proposed framework is implemented in a ROS-based simulator with a LiDAR equipped omnidirectional robot. Evaluation metrics for the trajectories generated by the proposed CBF synthesis framework quantify how well they match the ideal case where the CBF is known. To the best of our knowledge, this is the first paper addressing the problem of CBF synthesis from sensed environmental data.

This paper is organized as follows: Section~\ref{sec:math_bkg} reviews control barrier functions, their safety properties, and their use in QP-based control.  Section~\ref{sec:prob_stat} describes the problem addressed.  Section~\ref{sec:main_results} covers the main results of the CBF synthesis framework, for both the offline and online versions.  Section~\ref{sec:case_studies} covers implementation scenarios from a motion planning perspective along with evaluation metrics for comparing the generated trajectories with ground truth data.  Section~\ref{sec:conclusion} provides concluding remarks.

\section{MATHEMATICAL BACKGROUND}
\label{sec:math_bkg}
This section summarizes the concept of control barrier functions and the
formal safety guarantees they provide. To begin, consider an affine control robotic system: 
\begin{equation} \label{eq:system}
    \dot{x} = f(x) + g(x)u\,, \quad 
      x \in \mathcal{D} \subset \mathbb{R}^{n},\ 
      u \in \mathbb{R}^{m},
\end{equation}
where $x$ is the state of the robot, $u$ is the control input, and $x(0) = x_0$. Both
$f : \mathcal{D} \rightarrow \mathbb{R}^{n}$ and 
$g : \mathcal{D} \rightarrow \mathbb{R}^{n \times m}$ 
are locally Lipschitz continuous vector fields.

Consider further that the system has a set of safe states $\mathcal{C} = \{ x \in \mathcal{D} \mid h(x) \geq 0\,\text{ and } h \in C^1(\mathcal{D};\mathbb{R}) \}$ given by the super zero level-set of the function $h$.  The boundary of the safe set is the zero level-set, $\partial \mathcal{C} = \{ x \in \mathcal{D} \mid h(x) = 0 \}$.  During controlled evolution, the system \eqref{eq:system} is considered to be \emph{safe} if for all $t \geq 0$, $x(t) \in \mathcal{C}$ when $x(0) \in \mathcal{C}$. As detailed in \cite{CBFs_tutorial}, zeroing control barrier functions (ZCBFs) can be used to guarantee safety of the system. To define ZCBFs, we first define an extended class $\mathcal{K}$ function $\alpha : \mathbb{R} \rightarrow \mathbb{R}$ as a function that is strictly increasing and $\alpha(0) = 0$. 

\begin{definition}
\label{def:zcbf}
The function $h \in C^1(\mathcal{D};\mathbb{R})$ is a Zeroing Control Barrier Function (ZCBF) if there exists a locally Lipschitz extended class $\mathcal{K}$ function $\alpha$ such that for all $x \in \mathcal{D}$
\begin{equation*}
    \sup_{u \in \mathbb{R}^{m}}\bigg\{ L_{f} h(x) + L_{g} h(x) u(x) + \alpha(h(x)) \bigg\} \geq 0 \,,
\end{equation*}
for the Lie derivatives
$L_{f} h(x) = \frac{\partial h(x)}{\partial x} f(x)$ and 
$L_{g} h(x) = \frac{\partial h(x)}{\partial x} g(x)$ 
of $h$ in the direction of the vector fields $f$ and $g$.
\end{definition}

Define the state-dependent set of control inputs $\mathcal U(x)$,
\begin{equation} \label{eq:zcbf_controls}
    \mathcal{U}(x) \equiv \bigg\{ u \in \mathbb{R}^{m} \mid L_{f}h(x) + L_{g}h(x)u(x) + \alpha(h(x)) \geq 0 \bigg\} \ .
\end{equation}
Safety of the system can then be guaranteed under the action of a suitable control input $u(x) \in \mathcal{U}(x)$ for all $x \in \mathcal{D}$, formalized by the following theorem: 
\begin{theorem} \label{theorem:zcbf} \cite{CBFs_tutorial} 
  Let there be a safe set 
  $\mathcal{C} = \{ x \in \mathcal{D} \mid h(x) \geq 0 \text{ and } h \in
  C^1(\mathcal{D};\mathbb{R})\}$ specified for the affine control system
  \eqref{eq:system}. If $h$ is a ZCBF, then any control input
  $u \in C(\mathcal{D};\mathbb{R}^{m})$ where $u(x) \in
  \mathcal{U}(x)$ for all $x \in \mathcal{D}$ renders the set
  $\mathcal{C}$ forward invariant. That is, $x(t) \in \mathcal{C}$ 
  for all $t \geq 0$.
\end{theorem}

The constraint \eqref{eq:zcbf_controls} arising from a ZCBF $h$ 
is affine in the control input $u$.
and hence can be encoded as a quadratic
program (QP) constraint in $u$.  For fixed $x \in \mathcal{D}$, the
requirement $u \in \mathcal{U}(x)$ becomes a linear constraint for
the following point-wise in time, minimum norm QP:
\begin{equation}
\begin{aligned} \label{intro_qp}
  & \underset{u \in \mathbb{R}^{m}}{\text{minimize}}
  \quad ||u - k(x)||_{2}^{2} \\
  & \quad \text{s.t \quad \quad \quad} u \in \mathcal{U}(x) \ ,
\end{aligned}
\end{equation}
where $k : \mathcal{D} \rightarrow \mathbb{R}$ is a user-defined nominal control policy. This QP (a) results in a control input for following the nominal policy while simultaneously guaranteeing safety, and (b) is amenable to efficient online computation. 

\section{PROBLEM STATEMENT}
\label{sec:prob_stat}
Consider an affine control robotic system as in~\eqref{eq:system} evolving in $\mathcal{D} \subset \mathbb{R}^{2}$ and equipped with LiDAR sensors that provide depth information.  By virtue of the depth measurement vector $z_{t} \in \mathbb{R}_{>0}^{N}$ at time $t$, where $N$ is the total number of samples, the robot can detect unsafe state space regions.
Regarding the LiDAR sensor, denote by $\theta_{\text{res}}$ the angular resolution (increment angle) of the measurements. This is the angle between two consecutive light rays emitted from the sensor. We make the following assumption in order to account for spatial variations in the nature of the workspace: assume that the resolution of the LiDAR sensor is high enough to capture the spatial profile of the environment from a given offset distance, i.e., the LiDAR has a sufficiently small increment angle $\theta_{\text{res}}$.  Sensors such as the ones from Velodyne \cite{velodyne} with increment angles as small as $0.08^{\circ}$ are capable of satisfying the above assumption. Hence, it is reasonable to assume such sensor resolution capabilities.

Let $k \in C(\mathcal{D};\mathbb{R}^{m})$ be a user-defined nominal feedback control policy to be followed by the robot. Examples of such policies include proportional (go-to-goal) control 
or MPC based policies \cite{gurriet_active_set_invariance}. 
The state space is assumed to contain unknown unsafe regions.  That is, there exist $p$ unsafe sets in the state space defined as $\mathcal{C}_{i} = \{ x \in \mathcal{D} \mid h_{i}(x) \leq 0 \, h_{i} \in C(\mathcal{D};\mathbb{R})\}$ for all $i \in \{ 1,2,\ldots, p\}$, such that $h_i$ are unknown ZCBFs. The safe region is $\mathcal{D} \backslash \cup_{i=1}^{p}\!\mathcal{C}_{i}$.

Since there is no \textit{a priori} knowledge of the unsafe sets, data
obtained from the LiDAR sensor must be used to synthesize the
unknown barrier functions $h_{i}: \mathcal{D} \rightarrow \mathbb{R}$,
$i \in \{ 1,2,\ldots,p\}$, to render the system safe while minimally
deviating from the nominal feedback policy $k$. In conjunction with the
robot's state, the measurements obtained from the on-board depth sensors
provide the location of points on the boundary of the unsafe sets, and
hence are points $x \in \mathcal{D}$ for which $h(x) = 0$.  
To learn the unsafe regions and follow the nominal policy safely, a
framework for the synthesis of barrier functions is required with
guarantees on safety of the system, as formalized by the problem
statement:

\begin{probstat}
Consider the affine control robotic system in \eqref{eq:system} and the unsafe sets $\mathcal{C}_{i} \subset \mathcal{D}$, $i \in \{ 1,2,\ldots, p\}$.  Given the nominal feedback control policy $k : \mathcal{D} \rightarrow \mathbb{R}$ and LiDAR measurements $z_{t}$ obtained at any time instant $t \geq 0$, formulate a barrier function synthesis framework which either
\begin{enumerate}
    \item Learns the unsafe region $\bigcup_{i=1}^{p}\!\mathcal{C}_{i}$ offline given a dataset of safe and unsafe samples from the domain, or
    \item Learns the unsafe region online using instantaneous measurements $z_t$, as the system traverses the domain.
\end{enumerate}
\end{probstat}

\section{CONTROL BARRIER FUNCTION SYNTHESIS FRAMEWORK}
\label{sec:main_results}
This section describes the methodology for obtaining the training dataset, the control barrier function synthesis framework, and two QP based approaches which utilize the synthesized barrier function to guarantee safety.

\subsection{Support Vector Machines}
The learning approach to be used for barrier function specification via-a-vis the unsafe regions will be support vector machine (SVMs), namely kernel SVMs \cite{cristianini2000introduction}.  Suppose a dataset $\mathcal{T} = \{ (x_{1}, y_{1}), (x_{2}, y_{2}), \ldots, (x_{N}, y_{N})\}$ is provided where $x_{i} \in \mathbb{R}^{n}$ is an $n$ dimensional vector and $y \in \mathcal{Y} = \{ -1, 1\}$ is a label associated with the vector $x_i$ for all $i \in \{ 1,2,\ldots,N\}$.  Using the dataset, the linear SVM algorithm determines an affine decision boundary function $\widehat{f}(w^{T} x + b)$, where $x \in \mathbb{R}^{n}$ is a training sample, $w \in \mathbb{R}^{n}$ are coefficients and $b \in \mathbb{R}$ is a bias term, which translates the sample $x$ into a corresponding label $y \in \mathcal{Y}$ that belongs to one of the two classes i.e. $+1$ or $-1$. When the data is not separable by a hyperplane in the native space, a non-linear mapping transforming the data into a higher dimensional space with better separability properties may be used. This paper makes use of such a mapping, via a kernel function, to facilitate separation of unsafe obstacle regions from safe regions.

Since the domain $\mathcal{D}$ consists of states which are either safe or unsafe, their separation can be cast as a binary SVM classification problem. However, it is imperative that unsafe states be classified as unsafe, whereas all the safe states need not strictly be classified as safe. To that end, we consider the non-linear, biased-penalty SVM optimization problem \cite{veropoulos1999controlling}:
\begin{align} \label{eq:SVM_unbalanced_hardMargin}
\underset{w}{\text{minimize}} \quad 
& \frac{1}{2}||w||_{2}^{2} + C^{+} \sum\limits_{i \mid y_i = +1}^{N}\xi_{i} + C^{-}\sum\limits_{j \mid y_j = -1}^{N}\xi_{j} \nonumber \\
\text{s.t \quad} & y_{i} \cdot (w \phi(x_{i}) + b) \geq 1 - \xi_{i} \nonumber \\
& \xi_{i} \geq 0 \text{, for all } i \in \{ 1,2,\ldots, N\} \ ,
\end{align}
where $C^{+}, C^{-}> 0$ are constants penalizing misclassification of the positive and negative samples, and $\phi : \mathbb{R}^{n} \rightarrow \mathbb{R}^{d}$ is a non-linear mapping into a higher dimensional space. In practice, the dual of the above optimization problem is solved by using a kernel function $\kfun$ to bypass the need to explicitly define $\phi$ \cite{cristianini2000introduction}.  We use the Gaussian kernel, 
\begin{equation} \label{eq:rbf_kernel}
    \kfun(x_i,x_j) = 
      \exp{\bigg( - \frac{\| x_i - x_j \|^{2}}{\sigma^{2}} \bigg)}\,,
\end{equation}
where 
$\sigma > 0$ is the bandwidth of the kernel (and is a hyper parameter).

Observe that in \eqref{eq:SVM_unbalanced_hardMargin} there are two separate costs for the positive and negative classes. Unequal costs permit a greater bias towards correctly classifying one class over the other. In particular, having $C^{-} = \infty$ and $\infty > C^{+} > 1$ induces a hard margin classification for the unsafe states and allows for some misclassification for the safe states.  This outcome is captured by the so called \emph{cost matrix} ($M$) of the form
\begin{center}
  {\small
  \begin{tabular}{|c|c|c|}\hline
  \diagbox[width=10em]{\textbf{True}}{\textbf{Estimated}}&
  Safe & Unsafe \\ \hline
  Safe & $0$ & $C^{+}$ \\ \hline
  Unsafe & $C^{-}$ & $0$ \\ \hline
  \end{tabular}}
\end{center}
Each entry $[M]_{ij}$ of the matrix represents the cost of classifying a sample as label $j$ when it truly belongs to label $i$. The penalty for classifying a truly safe (or unsafe) state as safe (or unsafe) is zero.  It is undesirable to classify a truly unsafe state as safe, motivating a high penalty for $C^{-}$. Since safe states being classified as unsafe do not compromise safety, the penalty $C^{+}$ may be smaller. The optimization problem \eqref{eq:SVM_unbalanced_hardMargin} provides compliance (in favor of safety) to measurement errors and noise in the sensor data which can affect the generated decision boundary.  The mixed hard/soft margin classification is what supports the theoretical safety guarantees of the system as discussed in the following subsections.

\subsection{Training Dataset Generation}
This section details the training data generation process suited to binary SVM classification per \eqref{eq:SVM_unbalanced_hardMargin}. A pictorial example for obtaining the training data from a LiDAR sensor at a particular time instant is shown in Fig.~\ref{fig:Dataset_example}.  Below we provide a detailed explanation for generating the dataset.

Generating meaningful data for the kernel SVM from the LiDAR sensor requires converting the scalar variables into world Cartesian coordinates by means of a laser scan transform $g:\mathbb{R}\times\mathcal{D} \rightarrow \mathbb{R}^{2}$, whose main input is the laser scan measurements in polar coordinates and the current robot state (for mapping from the robot frame to the world frame).  Assume that if the sensor detects an unsafe region, then the output from the sensor is a finite depth reading, else it is infinite.  In particular, given a measurement vector $z_t = \begin{bmatrix} z_{t}^{1} & z_{t}^{2} & \ldots & z_{t}^{N} \end{bmatrix}^{T} \in \mathbb{R}^{N}$ at time $t$ with $N$ samples, define $\mathcal{F} \subset \mathcal{I} = \{1, \dots, N\}$ to be the index set of the finite scan measurements.  Define $\mathcal{O}^{-} = \bigcup_{i \in \mathcal{F}} \{ g(z_{t}^{i}; x_t) \}$ as the set of unsafe samples. $\mathcal{O}^{-}$ represents points on the boundary of the unsafe set detected by the sensor which is used to populate a dataset of negative labeled samples $\mathcal{T}^{-} = \bigcup_{i\in \mathcal{F}}\{(g(z_{t}^{i}; x_t), -1)\}$.  To obtain the positive samples from the environment, each $g(z_{t}^{i}; x_t) \in \mathcal{O}^{-}$ is projected radially backwards along the line segment joining the state of the robot $x(t)$ and the point $g(z_{t}^{i}; x_t)$, by a finite distance $d \in \mathbb{R}_{> 0}$.  Define
\begin{equation}
    \label{eq:pos_samples}
    \widehat{z}_{t}^{i} = g(z_{t}^{i} - d; x_t) \in \mathbb{R}^{2}
\end{equation}
for all $i \in \{1,2,\ldots,N\}$ where $d > 0$ is the finite offset
distance.  
Define the set of positive samples as $\mathcal{O}^{+} =
\bigcup_{i\in\mathcal{F}}\{\widehat{z}_{t}^{i}\}$, with the dataset for
positive labeled samples constructed as $\mathcal{T}^{+} =
\bigcup_{i\in\mathcal{F}}\{ (\widehat{z}_{t}^{i}, +1) \}$. 
Collecting the set of positive and negative labeled samples generates
the training dataset $\mathcal{T} = \mathcal{T}^{+} \cup \mathcal{T}^{-}$. 
The training dataset $\mathcal{T}$ contains all unsafe samples and
corresponding safe samples for training the SVM classifier. The
procedure is summarized in Algorithm~\ref{algo:training_dataset}.

\begin{figure}
    \centering
    \includegraphics[width=0.9\columnwidth]{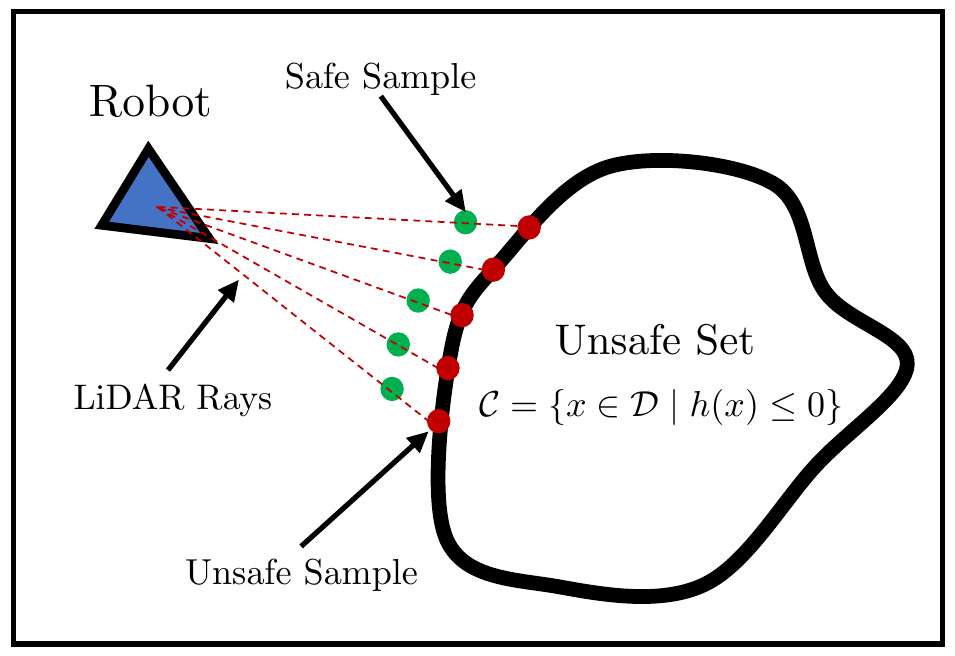}
    \caption{A particular instantiation of a training dataset obtained from measurements from a LiDAR sensor for a given unsafe set. The red points indicate unsafe samples which represent the boundary of the unsafe set whereas the green points indicate the safe samples obtained by the transformation dictated by equation \eqref{eq:pos_samples}. The red dashed lines indicate the LiDAR rays emanating from the sensor onboard the robot.}
    \label{fig:Dataset_example}
\end{figure}

\begin{algorithm}
\DontPrintSemicolon
\KwIn{Laser Scan Measurement $z_{t}$ and Robot State $x_t$}
\KwOut{Training Dataset $\mathcal{T}$}
\SetKwBlock{Begin}{function}{end function}
\Begin($\text{TrainingDataGenerator} {(}z_{t}{)}$)
{
Identify $\mathcal{F} \subset \{1, \dots, N\}$\;
$\mathcal{T}^{-} = \bigcup\limits_{i\in\mathcal{F}} \{ (g(z_{t}^{i}; x_t), -1) \}$ \;
$\widehat{z}_{t}^{i} = g(z_{t}^{i} - d; x_t)$, $\forall i \in \mathcal{F}$ \;
$\mathcal{T}^{+} = \bigcup\limits_{i\in\mathcal{F}} \{ (\widehat{z}_{t}^{i}, +1) \}$ \;
$\mathcal{T} \gets \mathcal{T}^{-} \cup \mathcal{T}^{+}$ \;
  \Return{$\mathcal{T}$}
}
\caption{Training Dataset Generator}
\label{algo:training_dataset}
\end{algorithm}

\subsection{Barrier Function Synthesis with Kernel-SVMs}
To improve the ability to capture unsafe region boundaries, the point data is
transformed by a fixed set of Gaussian kernels of the
form~\eqref{eq:rbf_kernel} using a sparse set of grid points over the domain
$\mathcal{D}$.  This provides a first kernel machine layer that behaves
like an approximate Hilbert space occupancy map \cite{RaOt_HilbertMaps}
and roughly captures the different safe and unsafe regions of the state space. Passing the vector output of
this Hilbert space to the kernel SVM generates a second layer 
that can refine the boundary to better separate the safe and unsafe regions.
The solution to the hard/soft margin kernel SVM in 
\eqref{eq:SVM_unbalanced_hardMargin}  defines the parameters for a
non-linear decision boundary separating the training data (the output
layer of the full classification network).  
Evaluating the two-layer classifier model for $x \in \mathcal{D}$ outputs a
posterior probability describing the likelihood that the sample $x \in
\mathcal{D}$ belongs to a particular class i.e., safe or unsafe. The
posterior probabilities obtained from the model are then converted into
margin scores which define a signed level-set function and provide the
barrier function we seek. The barrier function approximator is thus a two
hidden layer Gaussian kernel neural network.  
This entire procedure is summarized in
Algorithm~\ref{algo:SVM_barrier_synthesis}. By virtue of the methodology
used to generate the training data in
Algorithm~\ref{algo:training_dataset}, and the biased-penalty hard
margin SVM optimization problem \eqref{eq:SVM_unbalanced_hardMargin}, 
the synthesized barrier function correctly classifies the unsafe
samples. This is formalized in Proposition~\ref{prop:correct_predict}.

\begin{prop}
\label{prop:correct_predict}
Given a training dataset $\mathcal{T}$ obtained from Algorithm~\ref{algo:training_dataset}, if Algorithm~\ref{algo:SVM_barrier_synthesis} is used to synthesize the barrier function $\widehat{h}$, then the unsafe samples $\mathcal{O}^{-}$ are such that $\widehat{h}(x) < 0$ for all $x \in \mathcal{O}^{-}$.
\end{prop}
\begin{proof}
By the method presented in Algorithm~\ref{algo:training_dataset} to generate the training dataset $\mathcal{T}$, we have that the set $\mathcal{O}^{-}$ consists of points on the boundary of the unsafe set. From the kernel-SVM approach used in Algorithm~\ref{algo:SVM_barrier_synthesis}, a function $\widehat{h}$ is generated which classifies the \emph{safe} and \emph{unsafe} samples. Since the optimization problem \eqref{eq:SVM_unbalanced_hardMargin} is a hard margin SVM for the unsafe samples and RBF kernels have universal function approximation capabilities (Theorem 2, \cite{hammer2003_svm_approx}), we can guarantee that $\widehat{h}(x) < 0$ for all $x \in \mathcal{O}^{-}$ and thus the proposition follows.
\end{proof}

\begin{algorithm}
\DontPrintSemicolon
\KwIn{Training Dataset $\mathcal{T}$}
\KwOut{Estimated Barrier Function $\widehat{h}$}
\SetKwBlock{Begin}{function}{end function}
\Begin($\text{BarrierEstimator} {(}\mathcal{T}{)}$)
{
$\mathcal{T}_{HS} \gets $ Map samples in $\mathcal{T}$ to Hilbert space \;
${Cl} \gets $ Train kernel SVM classifier
\eqref{eq:SVM_unbalanced_hardMargin} using $\mathcal{T}_{HS}$ \;
$\widehat{h} \gets $ Recover signed distance function from ${Cl}$ and
first Gaussian kernel layer mapping\;
  \Return{$\widehat{h}$}
}
\caption{Kernel-SVM based Barrier Function Synthesis}
\label{algo:SVM_barrier_synthesis}
\end{algorithm}

\subsection{Offline Barrier Function Synthesis \& Control}
Here, we discuss the \textit{offline} approach to CBF synthesis using 
Algorithm~\ref{algo:training_dataset} and Algorithm~\ref{algo:SVM_barrier_synthesis}. Per the problem setup in Section~\ref{sec:prob_stat}, we consider 
the workspace consisting of $p$ unsafe regions characterized by ZCBFs
$h_{i}$, $i \in \{ 1,2,\ldots,p\}$. We assume that there exists
an \textit{oracle} which provides a set of unsafe samples corresponding
to the boundary of each unsafe sets $i \in \{1,2,\ldots,p \}$ in the
state space by means of a LiDAR sensor which are dense enough to cover
the true boundary of the obstacles. For example, this oracle can be a
``mapping'' robot that navigates the domain and gathers data about the
safe and unsafe regions.

Once Algorithm~\ref{algo:training_dataset} generates the requisite
training data using the \textit{oracle}, executing
Algorithm~\ref{algo:SVM_barrier_synthesis} leads to a ZBF estimate. 
Note that a single ZCBF $\widehat{h} \in C^1(\mathcal{D};\mathbb{R})$, 
whose zero level-set captures the boundaries between safe and unsafe regions, 
is obtained as opposed to $p$ different ZCBFs characterizing the unsafe sets. 
With the synthesized barrier function $\widehat{h}$, we then implement a
QP controller with~\eqref{eq:zcbf_controls} as the constraint. 
Capturing all the unsafe sets with a single function means that the
QP involves only one constraint which reduces the computational
complexity involved in computing the control input. The QP is solved,
and the control is applied, until the system completes the specified
task associated to the nominal controller.  The entire offline barrier
function synthesis and control methodology is formalized in
Algorithm~\ref{algo:offline_SVM}.  In the algorithm, the initial loop
from $t = 0$ to $t = T$ where $T < \infty$, indicates the time period
when the training data is gathered for generating the barrier function. 

\begin{algorithm}
\DontPrintSemicolon
\KwIn{Nominal controller $k$}
\SetKwBlock{Begin}{function}{end function}
  $\mathcal{T} \gets \emptyset$ \;
  \ForAll{$t \in [0,T]$}
  {
  $z_{t} \gets $ \texttt{LaserScanMeasurement} \;
  $\mathcal{T}_{t} \gets $ \texttt{TrainingDataGenerator}$(z_{t},x_t)$ \;
  $\mathcal{T} \gets \mathcal{T} \cup \mathcal{T}_{t}$ \;
  }
  $\hat{h} \gets $ \texttt{BarrierEstimator} ($\mathcal{T}$) \;
  \While{\texttt{Goal is not reached}}
  {
    Solve the QP:
    \begin{equation*}
    \begin{aligned}
    & \underset{\quad \quad \quad \quad u \in \mathbb{R}^{m}}{\text{$u^{*}(x)$ = argmin}}
    \quad ||u - k(x)||_{2}^{2}\\
    & \text{\quad \quad \quad \quad \quad s.t \quad} L_{f} \widehat{h}(x) + L_{g} \widehat{h}(x) u(x) \geq -\alpha(\widehat{h}(x))
    \end{aligned}
    \end{equation*}
    $u \gets u^{*}(x)$ \;
    Solve \eqref{eq:system}, update state $x(t)$ \;
  }
\caption{Offline SVM-based QP controller}
\label{algo:offline_SVM}
\end{algorithm}


Recall that the increment angle of the LiDAR sensor is given by $\theta_{\text{res}}$. Intuitively, as $\theta_{\text{res}} \rightarrow 0$, the LiDAR sensor captures the true nature of the boundary of the unsafe region. Hence, using Proposition~\ref{prop:correct_predict}, we can guarantee that Algorithm~\ref{algo:SVM_barrier_synthesis} synthesizes a barrier function whose level-sets are over-approximations of the true unsafe regions. That is, denote $\widehat{S} = \{ x \in \mathcal{D} \mid \widehat{h}(x) \leq 0 \}$ where $\widehat{S} : \mathcal{D} \rightarrow \mathbb{R}$ as the unsafe region estimated by Algorithm~\ref{algo:SVM_barrier_synthesis}. Then, we have that $\mathcal{S} \subset \widehat{\mathcal{S}}$, where $\mathcal{S} = \bigcup\limits_{i = 1}^{p} \{ x \in \mathcal{D} \mid h_{i}(x) \leq 0\}$ is the true unsafe region characterized by the unknown barrier functions $h_{i}$ for all $i \in \{ 1,2,\ldots,p\}$. In practice, this statement holds true for high resolution LiDAR sensors. The degree of over-approximation depends on a number of factors which include the distance $d \in \mathbb{R}_{>0}$ with which the positive samples are generated in Algorithm~\ref{algo:training_dataset}. Next, we provide a formal guarantee that Algorithm~\ref{algo:offline_SVM} guarantees safety of the system.

\begin{theorem}
\label{theorem:safety_guarantee}
Suppose $\mathcal{S} \subset \widehat{S}$ and the controller from Algorithm 3 is used. Then given any $x(0) \in \widehat{\mathcal{S}}^{c}$ where $\widehat{\mathcal{S}}^{c} = \{ x \in \mathcal{D} \mid \widehat{h}(x) \geq 0 \}$, the robot trajectory is such that $x(t) \in \widehat{\mathcal{S}}^{c}$ for all $t \geq 0$.
\end{theorem}
\begin{proof}
From Algorithm~\ref{algo:offline_SVM}, the QP enforces the barrier function constraint \eqref{eq:zcbf_controls} with $\widehat{h}$ as the ZCBF. Since the cost function of the QP is quasi-convex in $u$, the constraints are quasi-convex in $u$ and the nominal policy $k$ is continuous, from Proposition 8 in \cite{mcbfs} we have that the generated control $u$ is continuous. Hence from Theorem~\ref{theorem:zcbf} and by assumption $\mathcal{S} \subset \mathcal{\widehat{S}}$, we have that the set $\widehat{\mathcal{S}}^{c} = \{ x \in \mathcal{D} \mid \widehat{h}(x) \geq 0 \}$ is rendered forward invariant. That is, we have that $x(t) \in \widehat{\mathcal{S}}^{c}$ for all $t \geq 0$.
\end{proof}

\begin{figure*}[t]
\centerline{ 
\subcaptionbox{\label{fig2:a}}{\includegraphics[width=\columnwidth, height=0.65\columnwidth]{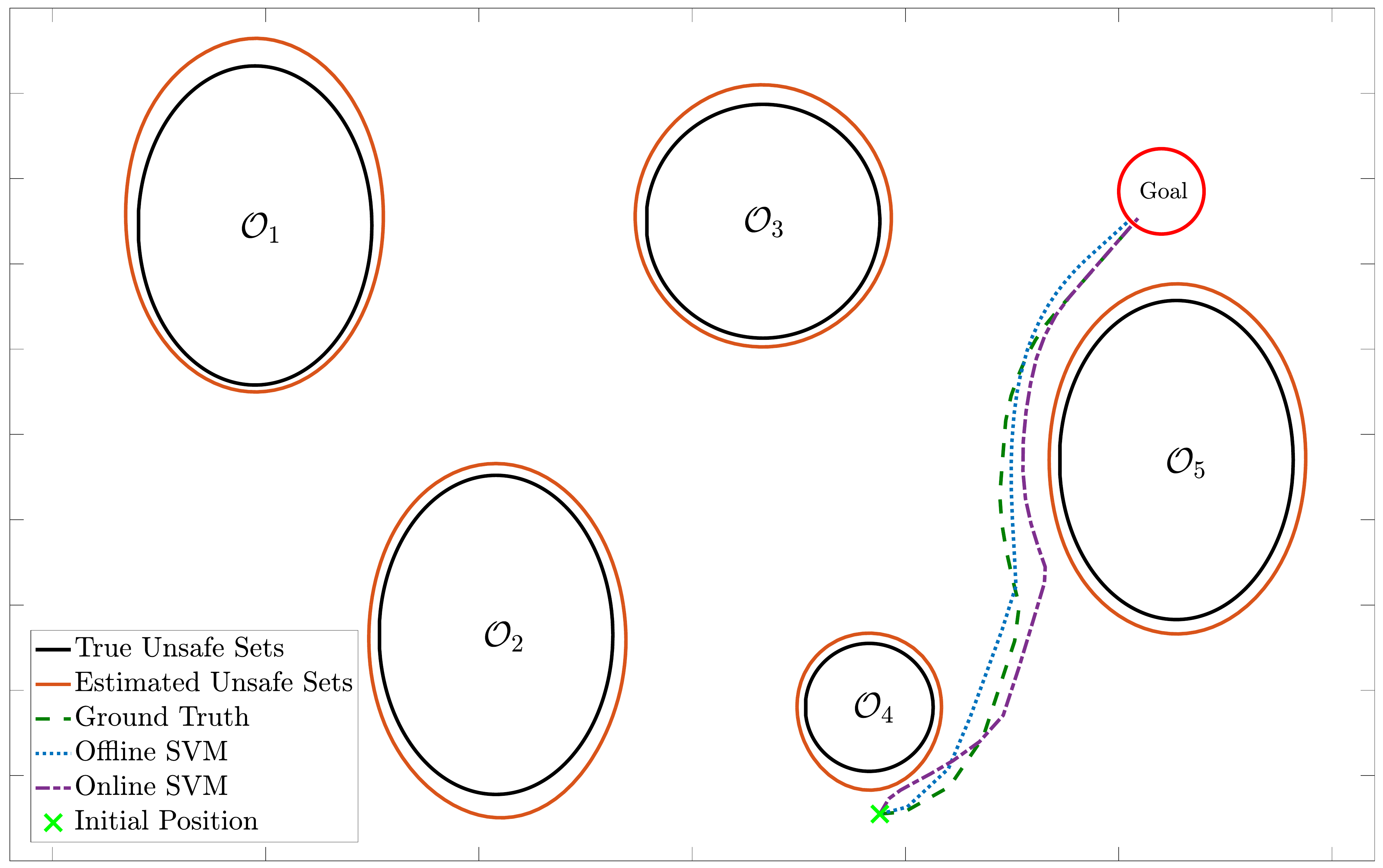}}~
\subcaptionbox{\label{fig2:b}}{\includegraphics[width=\columnwidth, height=0.65\columnwidth]{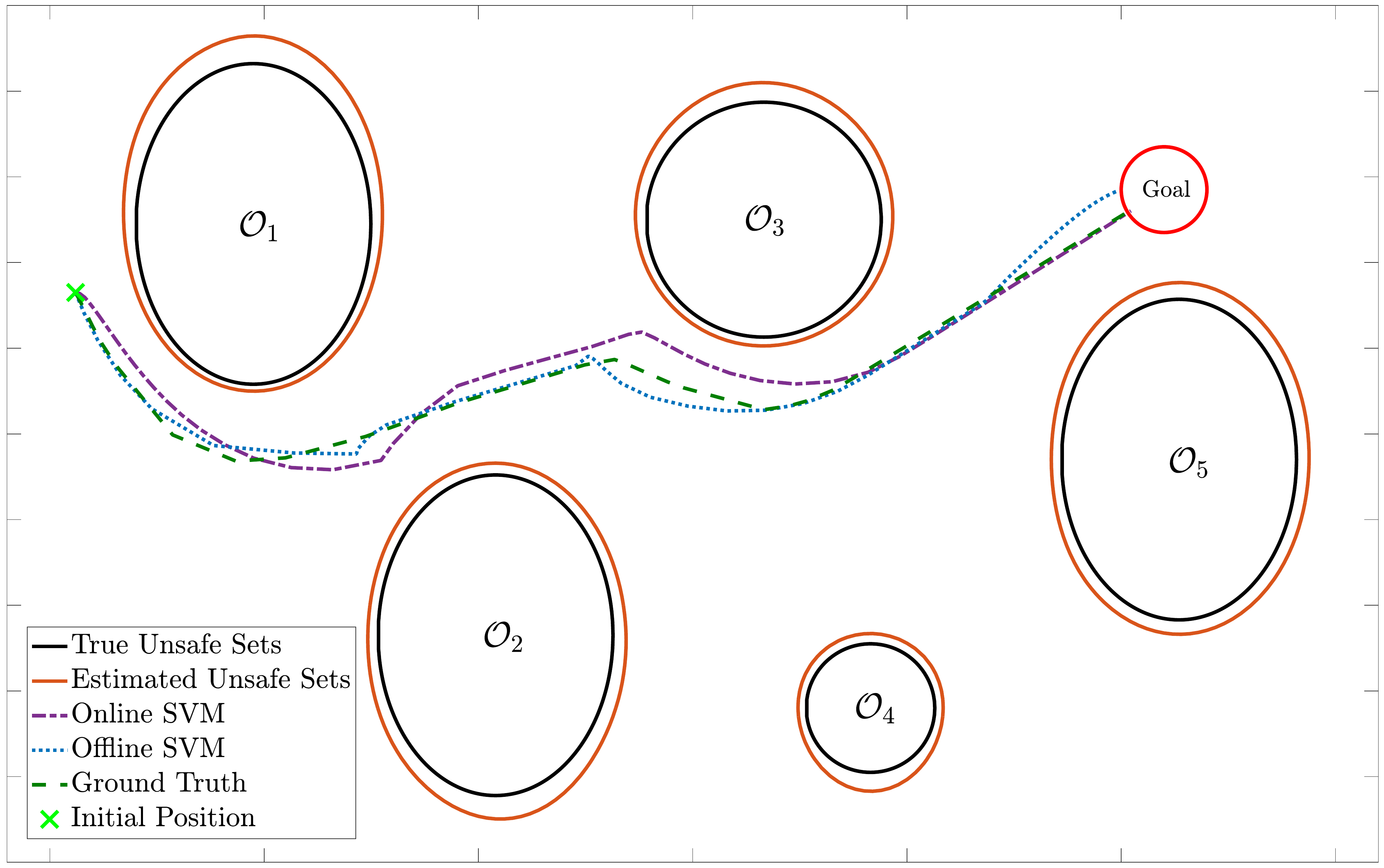}}
}
\caption{Trajectories generated for the robot in a five obstacle scenario. The robot must reach a goal region (red circle) which is known a priori. Three different trajectories are shown- the ground truth trajectory (dashed green), the offline kernel-SVM based controller trajectory (dotted blue), and the online kernel-SVM based controller trajectory (dash-dotted purple). For the initial condition on the left, the trajectories show high correlation values ($R_{\text{offline}} = 0.9992$, $R_{\text{online}} = 0.9777$, $R_{\text{offline-online}} = 0.9734$) and small Fr\'echet distance values ($F_{\text{offline}} = 0.0469$, $F_{\text{online}} = 0.0822$, $F_{\text{offline-online}} = 0.0853$) which indicate that the trajectories are highly similar to the ground truth trajectory. For the figure on the right, the trajectories once again show high correlation values ($R_{\text{offline}} = 0.9627$, $R_{\text{online}} = 0.8085$, $R_{\text{offline-online}} = 0.8946$) and small Fr\'echet distance values ($F_{\text{offline}} = 0.0665$, $F_{\text{online}} = 0.0840$, $F_{\text{offline-online}} = 0.1334$). Observe that estimated unsafe set is an over-approximation of the true unsafe sets, and hence Algorithm~\ref{algo:offline_SVM} guarantees collision free trajectories in the offline case, as per Theorem~\ref{theorem:safety_guarantee}.}
\label{fig:CaseStudy1}
\end{figure*}

\subsection{Online Barrier Function Synthesis \& Control}
\begin{algorithm}
\DontPrintSemicolon
\KwIn{Aggregate Flag $\delta$, Nominal controller $k$}
\SetKwBlock{Begin}{function}{end function}
  $\mathcal{T} \gets \emptyset$ \;
  \While{\texttt{Goal is not reached}}
  {
    $z_{t} \gets $ \texttt{LaserScanMeasurement} \;
    $\mathcal{T}_{t} \gets $ \texttt{TrainingDataGenerator}$(z_{t}$) \;
    \uIf{$\delta = 1$}
    {
    $\mathcal{T} \gets \mathcal{T} \cup \mathcal{T}_{t}$ \;
    }
    \Else
    {
    $\mathcal{T} \gets \mathcal{T}_{t}$ \;
    }
    $\hat{h} \gets $ \texttt{BarrierEstimator} ($\mathcal{T}$) \;
    Solve the QP:
    \begin{equation*}
    \begin{aligned}
    & \underset{\quad \quad \quad \quad u \in \mathbb{R}^{m}}{\text{$u^{*}(x)$ = argmin}}
    \quad ||u - k(x)||_{2}^{2}\\
    & \text{\quad \quad \quad \quad \quad s.t \quad} L_{f} \widehat{h}(x) + L_{g} \widehat{h}(x) u(x) \geq -\alpha(\widehat{h}(x))
    \end{aligned}
    \end{equation*}
    $u \gets u^{*}(x)$ \;
    Solve \eqref{eq:system}, update state $x(t)$ \;
  }
\caption{Online SVM-based QP controller}
\label{algo:online_SVM}
\end{algorithm}

When access to the full set of unsafe samples from the environment is not available, a real-time barrier function synthesis method is preferable. Here, we describe an online approach to synthesizing barrier functions, based on Algorithm~\ref{algo:online_SVM}.  For online ZCBF synthesis, the set of unsafe samples covering the boundary of all the unsafe regions is not known a priori. Hence, at time $t = 0$, the system is initialized with no information regarding the state space, except the nominal feedback control policy. At each time instant $t$, the system obtains the depth measurement $z_t$ and generates the training dataset $\mathcal{T}$ via Algorithm~\ref{algo:training_dataset}.  Then, Algorithm~\ref{algo:SVM_barrier_synthesis} synthesizes a local barrier function. Implementing the QP controller generates the control input at time instant $t$. In the next time instant, the same procedure repeats and a new barrier function is synthesized based on the updated sensor measurements.

Two variations of the online barrier function synthesis method can be implemented. In the first method, the depth sensor data for all previous time instances is deleted, and the QP is solved with only the immediately sensed measurements.  The barrier function approximates the true safe region only locally i.e., in a neighborhood around the state $x_{t}$ of the robot. In the second method, samples from the previous time instant are aggregated with the samples from the current time instant, with Algorithm~\ref{algo:SVM_barrier_synthesis} implemented with the incremented set. The two cases synthesize different barrier function at each time instant. For the data-aggregation case, the estimate of the barrier improves as the number of samples characterizing the unsafe regions increases.  Advantages and drawbacks exist for both approaches.  In the data aggregation case, one needs to continuously update the dataset with new measurements and this exhaustive data collection process can become computationally expensive unless one resorts to efficient ways to store data \cite{dagger}. For the non data aggregation case, computation is faster but the estimate of the barrier function does not improve iteratively as the robot traverses the domain.  Both procedures are described in Algorithm~\ref{algo:online_SVM}.

Define the sensing range of the sensor as $\mathcal{B}_{r}(x) = \{ \overline{x} \in \mathcal{D} \mid \left\| x - \overline{x} \right\| \leq r \}$, where $r \in \mathbb{R}_{> 0}$ is the sensing range of the robot.  Similar to the discussion in the previous subsection, it can be guaranteed that if $\theta_{\text{res}} \rightarrow 0$, then locally, Algorithm~\ref{algo:SVM_barrier_synthesis} synthesizes a barrier function whose level-set over approximates the true unsafe region. That is, denote $\widehat{\mathcal{S}_{r}}(x) = \{ x \in \mathcal{B}_{r}(x) \mid \widehat{h}(x) \leq 0\}$ where $\widehat{h} : \mathcal{D} \rightarrow \mathbb{R}$ is the estimated ZCBF from Algorithm~\ref{algo:SVM_barrier_synthesis}. Then, as $\theta_{\text{res}} \rightarrow 0$, we have that $\mathcal{S}_{r}(x) \subset \widehat{\mathcal{S}}_{r}(x)$ for all $x \in \mathcal{D}$ locally within the ball $\mathcal{B}_{r}(x)$, where $\mathcal{S}_{r}(x) = \bigcup\limits_{i=1}^{p}\{ \overline{x} \in \mathcal{B}_{r}(x) \mid h_{i}(\overline{x}) \leq 0 \}$ is the true unsafe region.
In the online case, a statement similar to Theorem~\ref{theorem:safety_guarantee} cannot be made since the robot does not have access to the full set of samples that characterize the entire boundary of the unsafe set and hence, there is no guarantee that globally in the domain the generated level-sets are over-approximations of the true unsafe regions. However, since the robot dynamics are locally Lipschitz continuous, safety holds locally as seen in Fig~\ref{fig:CaseStudy1}.

\section{EXPERIMENTAL RESULTS}
\label{sec:case_studies}
This section describes and discusses simulation results from a path planning perspective conducted on the ``Simple Two Dimensional Robot (STDR) simulator"\footnote{\url{http://wiki.ros.org/stdr_simulator}}.
Two environments were created for use in STDR.  The first environment contains five ellipsoidal obstacles scattered throughout a 3.2 x 2 workspace domain.  The second environment of the same size contains more general obstacles whose shape cannot be characterized easily by level-sets of closed-form polynomials.  In both cases, the robot has no a priori knowledge of the environment and follows a nominal controller that drives it towards a goal point.  More formally, we consider a robot with dynamics $\dot{x} = u$, where $x \in \mathcal{D} \subset \mathbb{R}^{2}$ is the position of the robot and $u \in \mathbb{R}^{2}$ is the control input. The nominal feedback control policy for all $x \in \mathcal{D}$ is given by $k(x) = \delta \cdot \frac{(x - x_{\text{goal}})}{\mid \| x - x_{\text{goal}} \mid \|}$, where $\delta \in \mathbb{R}_{> 0}$, and $x_{\text{goal}} \in \mathcal{D}$ is a desired final goal position for the robot. Informally, the robot must follow $k(x)$ as close as possible while avoiding the unknown obstacles in the workspace.  The robot must reach a goal region which is defined as $\mathcal{G} = \{ x \in \mathcal{D} \mid \| x - x_{\text{goal}} \| \leq 0.1 \}$. For the first scenario, depicted in Fig.~\ref{fig:CaseStudy1}, we obtain ground truth data using a grid-based solution, which is a common approach to compute the true signed distance to the obstacles. The signed distance function corresponds to the true barrier function characterizing the obstacles.

\subsection{Evaluation Metrics}
Comparison of the trajectory outcomes for the different implementations involves two evaluation metrics: the correlation coefficient ($R$) and the Fr\'echet distance ($F$). These metrics capture both the evolutionary mismatch between trajectories, as well as the Euclidean distance mismatch. The combination of both these metrics provides a means to evaluate the outcomes of the proposed algorithms.

\subsubsection{Correlation Coefficient}
Informally, the correlation coefficient between two trajectories captures the change in one trajectory with respect to the other. That is, one can obtain information regarding the flow of one trajectory with respect to the other. 
Typically, two trajectories are said to be highly correlated if they have a correlation coefficient greater than $0.7$ \cite{bmi_corr}. 
We make use of the correlation coefficient to develop an intuition regarding the nature of the trajectories generated by the offline and online kernel-SVM based approaches compared with the ground truth data.

\subsubsection{Fr\'echet Distance}
Informally, the Fr\'echet distance provides a measure of the Euclidean distance mismatch between two trajectories. While the correlation coefficient provides information regarding the flow of two trajectories, the Fr\'echet distance provides an explicit degree of mismatch between the two. 
A lower Fr\'echet distance indicates less mismatch between the two
trajectories. In particular, $F = 0$ implies that the two trajectories
are identical. 

\subsection{Implementation Results}
\begin{figure}
    \centering
    \includegraphics[width=\columnwidth]{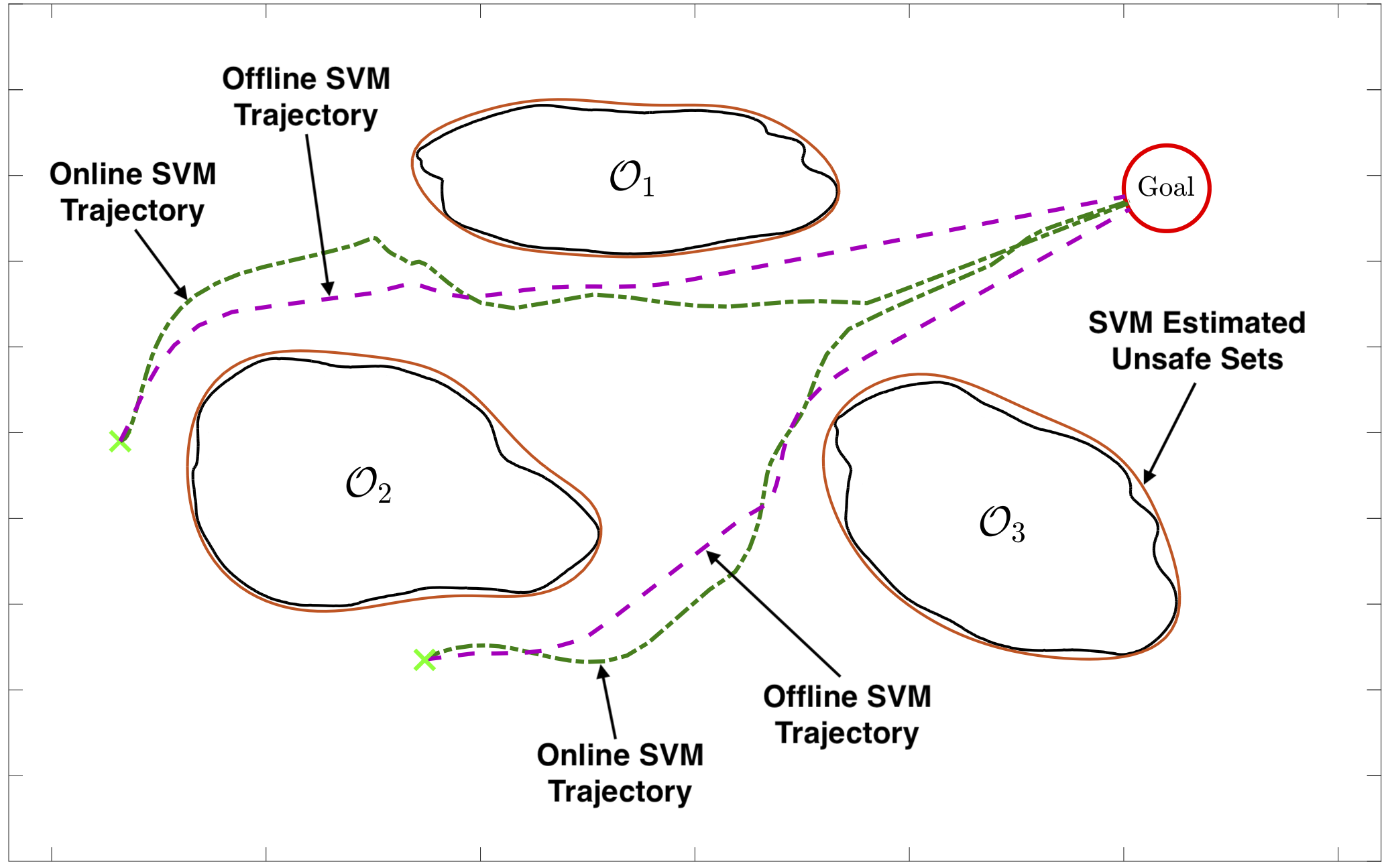}
    \caption{An implementation in the STDR simulator where the robot has to navigate the unknown environment to reach a goal region (red circle). Offline kernel-SVM based controller and online kernel-SVM based controller trajectories for two different initial conditions (green crosses) are shown. The obstacles $\mathcal{O}_{1}$, $\mathcal{O}_{2}$ and $\mathcal{O}_{3}$ are such that they cannot be easily characterized by closed form polynomials, and hence, using the traditional CBF formulation is difficult. However, using Algorithm~\ref{algo:offline_SVM} and Algorithm~\ref{algo:online_SVM}, we can generate trajectories such that the robot remains safe. \label{fig:STDR_squiggle}}
\end{figure}

We first consider the five obstacle scenario shown in
Fig.~\ref{fig:CaseStudy1}. Two different initial conditions for the
robot are considered. Three different trajectories are plotted in each
figure. The green dashed trajectory indicates the ground truth
trajectory obtained when the barrier function for each obstacle is known
a priori. A QP of the form~\eqref{intro_qp} is solved to generate this
trajectory. The blue, dotted trajectory is generated from the offline
kernel-SVM based barrier estimation approach as discussed in
Algorithm~\ref{algo:offline_SVM}. The purple, dash-dotted trajectory is
generated using Algorithm~\ref{algo:online_SVM} which is the online
kernel-SVM based barrier function estimation method. Observe that in
both the cases, the robots avoid the obstacle and follow the nominal
control policy as close as possible. In the second scenario, we consider
a situation where the obstacle shapes are such that finding the closed
form expressions for the barrier functions is not straightforward. This
setting is as shown in Fig.~\ref{fig:STDR_squiggle}. The pink, dashed
trajectories are generated using the offline kernel-SVM based barrier
function approach as discussed in Algorithm~\ref{algo:offline_SVM},
whereas the green, dash-dotted trajectories are generated using the online kernel-SVM based barrier function method described in Algorithm~\ref{algo:online_SVM}. A video of the simulations results is also provided\footnote{\url{https://youtu.be/-XiaR7QchtQ}}.

\addtolength{\textheight}{-6cm}
\subsection{Discussion \& Future Work}

\begin{table}[t]
  \centering
    \caption{Correlation Coefficients for Five Obstacle Scenario (Values close to
    $1$ indicate high correlation) \label{table:correlation_5Obs}}
    \begin{tabular}{c|c|c|c}
      \hline
      \textbf{Case} & \textbf{Offline SVM} & \textbf{Online SVM} & \textbf{Offline SVM}\\
      & \textbf{vs Ground Truth} & \textbf{vs Ground Truth} & \textbf{vs Online SVM} \\
      \hline
         1   & 0.9992 & 0.9777 & 0.9734 \\
         2   & 0.9627 & 0.8085 & 0.8946 \\
         3   & 0.9997 & 0.9709 & 0.9694 \\
         4   & 0.9889 & 0.9466 & 0.9195 \\ 
         5   & 0.9954 & 0.9442 & 0.9447 \\
         6   & 0.9991 & 0.9882 & 0.9870 \\
         7   & 0.9800 & 0.9865 & 0.9811 \\
         8   & 0.9692 & 0.7601 & 0.5886 \\ 
         9   & 0.9880 & 0.8718 & 0.8665 \\
         10   & 0.9997 & 0.9899 & 0.9874 \\
         \textbf{Average}   & \textbf{0.9881} & \textbf{0.9244} & \textbf{0.9112} \\\hline
    \end{tabular}
    \caption{Fr\'echet Distance for Five Obstacle Scenario (Smaller
    values indicate less mismatch between trajectories)\label{table:frechet_5Obs}}
    \begin{tabular}{c|c|c|c}
    \hline
      \textbf{Case} & \textbf{Offline SVM} & \textbf{Online SVM} & \textbf{Offline SVM}\\
      & \textbf{vs Ground Truth} & \textbf{vs Ground Truth} & \textbf{vs Online SVM} \\
      \hline
         1   & 0.0469 & 0.0822 & 0.0853 \\
         2   & 0.0665 & 0.0840 & 0.1334 \\
         3   & 0.0276 & 0.0446 & 0.0468 \\
         4   & 0.0582 & 0.1444 & 0.1232 \\
         5   & 0.0660 & 0.1563 & 0.1341 \\
         6   & 0.0308 & 0.0392 & 0.0266 \\
         7   & 0.1197 & 0.1296 & 0.0479 \\
         8   & 0.0496 & 0.0759 & 0.0652 \\
         9   & 0.0578 & 0.1368 & 0.1119 \\
         10   & 0.0389 & 0.0327 & 0.0369 \\
         \textbf{Average} & \textbf{0.0562} & \textbf{0.0925} & \textbf{0.0811} \\\hline
    \end{tabular}
\end{table}

Table~\ref{table:correlation_5Obs} compares the correlation coefficient for both the online and offline approaches against the ground truth trajectory in the first scenario. In addition, the online method is also compared to the offline case. On average, we obtain correlation coefficient values $> 0.90$, which shows a high similarity between the ground truth trajectory and the barrier estimated trajectory. In particular, note that the average correlation between the offline kernel-SVM approach and the ground truth trajectory is greater then $0.98$. We then provide Fr\'echet distances which measures the degree of mismatch in terms of the Euclidean distance between two 2D trajectories. The smaller the Fr\'echet distance, the smaller the mismatch between the two trajectories. Table~\ref{table:frechet_5Obs} shows the Fr\'echet distances between trajectories for the five obstacle scenario. Observe that on average, we obtain distances $<0.10$ for each case, which shows that the Euclidean distance mismatch between the trajectories is small. A key inference from the above data is that $R_{\text{offline}}$ is very high and $F_{\text{offline}}$ is very small, which shows that the offline kernel-SVM estimated barrier function closely replicates the true barrier functions.

A direction of future research is to extend the proposed SVM-based learning technique for synthesizing CBFs to other sensor models besides LIDAR such as RGB cameras. This could be done by identifying the depth map from a stereo image and then using that to generate the training data and the barrier function.

\section{CONCLUDING REMARKS}
\label{sec:conclusion}
This paper presented a supervised machine learning based approach to
automated synthesis of control barrier functions. A kernel-SVM based
method classifies the set of safe and unsafe samples, and generates the
desired barrier (level-set) function. A formal guarantee on zero
misclassification of unsafe samples is provided along with guarantees on safety of the robot. The proposed framework
was evaluated based on the comparison between the generated trajectories
and ground truth data. Experimental simulations using the proposed
framework were conducted on an omnidirectional robot in a ROS-based
simulator using synthetic LiDAR data.

\section*{Acknowledgement}
The authors thank Alex Chang for discussions regarding SVMs and for
helping with the initial code base.

\bibliographystyle{ieeetr}
\bibliography{bib}

\end{document}